\documentclass[final,12pt]{colt2026} 

\newcommand{\bbR}{\mathbb{R}}
\newcommand{\Tr}{\operatorname{Tr}}
\newcommand{\cP}{\mathcal{P}}
\newcommand{\cS}{\mathcal{S}}
\newcommand{\cB}{\mathcal{B}}
\newcommand{\cH}{\mathcal{H}}
\newcommand{\cC}{\mathcal{C}}
\newcommand{\cN}{\mathcal{N}}
\newcommand{\dR}{\mathbb{R}}
\newcommand{\dE}{\mathbb{E}}
\newcommand{\dP}{\mathbb{P}}
\usepackage{bbm}
\DeclareMathOperator*{\Argmin}{Arg\,min}
\DeclareMathOperator*{\Argmax}{Arg\,max}
\newcommand\numberthis{\addtocounter{equation}{1}\tag{\theequation}}
\title[Convex Relaxations for Graph Alignment]{Phase Transition in Convex Relaxations for Graph Alignment}
\usepackage{times}
\usepackage[dvipsnames]{xcolor}
\newcommand{\rev}[1]{#1}
\coltauthor{%
 \Name{Laurent Massouli\'e} \Email{laurent.massoulie@inria.fr}\\
 \addr INRIA, DI/ENS, PSL Research University, Paris, France
 \AND
 \Name{Sushil Mahavir Varma} \Email{sushilv@umich.edu}\\
 \addr Industrial and Operations Engineering, University of Michigan, Ann Arbor, Michigan, USA.
 \AND
 \Name{Louis Vassaux}
 \Email{louis.vassaux@inria.fr }\\
 \addr INRIA, DI/ENS, PSL Research University, Paris, France
 \AND
 \Name{Ir\`ene Waldspurger}
 \Email{waldspurger@ceremade.dauphine.fr} \\
 \addr CNRS, INRIA, Universit\'e Paris Dauphine, Paris, France
}
\begin{document}
\maketitle
\begin{abstract}
We study the graph alignment problem for correlated Gaussian Orthogonal Ensemble (GOE) matrices, where the goal is to recover a hidden vertex permutation given two correlated symmetric Gaussian matrices $(A, B)$ with correlation $1/\sqrt{1+\sigma^2}$. While the maximum likelihood estimator is information-theoretically optimal, its computation, which reduces to a quadratic assignment problem, is intractable. Motivated by this, we analyze convex relaxations based on minimizing $\|AX - XB\|_F$ over the set of doubly stochastic matrices and the unit hypercube. We show that when the correlation parameter satisfies $\sigma = o(n^{-1/2}/\log^4 n)$, the solution of either relaxation $(X^\star)$ concentrates around the ground-truth permutation matrix $(\Pi^\star)$, i.e., $\|X^\star-\Pi^\star\|_F^2 = o(n)$, implying recovery of all but a vanishing fraction of vertices after simple post-processing. Combined with existing lower bounds, our results precisely characterize that $\|X^\star-\Pi^\star\|_F^2$ transitions from $o(n)$ for $\sigma = \tilde{o}(n^{-1/2})$ to $\Omega(n)$ for $\sigma = \tilde{\Omega}(n^{-1/2})$. In doing so, our analysis significantly tightens prior results and extends them beyond doubly stochastic relaxations.
\end{abstract}
\section{Introduction}
The graph alignment problem is defined as finding a mapping between the vertices of two undirected graphs $G_1$ and $G_2$ such that the edge overlap is maximized. This problem is popular because it arises in widespread applications, e.g., network de-anonymization \cite{narayanan2008robust}, computational biology \cite{singh2008global}, pattern recognition \cite{conte2004thirty}, etc. More formally, let $A, B \in \bbR^{n \times n}$ be the adjacency matrices (possibly weighted) of $G_1$ and $G_2$ respectively. The objective is to find a permutation $\pi:[n] \to [n]$ such that the edge overlap $\sum_{i,j\in[n]} A_{i,j}B_{\pi(i), \pi(j)}$ is maximized.
We consider the problem of graph alignment, where the inputs are sampled as correlated Gaussian random matrices $A$ and $B$. In particular, we first sample two independent and identically distributed (i.i.d.) Gaussian Wigner matrices $A$ and $Z$. That is, $Z$ is a symmetric matrix, $\{Z_{i,j}\}_{i \leq j}$ are independent, $Z_{i, i}$ are i.i.d. Gaussian with $0$ mean and $2/n$ variance for $i \in [n]$, and $Z_{i, j}$ are i.i.d. Gaussian with $0$ mean and $1/n$ variance for $i < j \in [n]$.
To define $B$, we first sample $\pi^*$ uniformly at random from the set $\cS_n$ of permutations of $[n]$. We finally let
$$
B^\prime=A+\sigma Z,
$$
and define
$$
B_{i,j }=B^\prime_{\pi^\star(i),\pi^\star(j)},\quad i\le j\in [n].
$$
The objective is to then infer the hidden permutation $\pi^\star$, given the correlated Gaussian matrices $(A, B)$. It is known that one can formulate the maximum likelihood estimation as a quadratic assignment problem (QAP) given by $\tilde{\Pi}^\star = \arg\min_{\Pi \in \cP_n} \|A\Pi - \Pi B\|_F$, where $\cP_n$ is the set of all permutation matrices defined as
\begin{align*}
    \cP_n = \left\{\Pi \in \{0, 1\}^{n \times n}: \sum_{i=1}^n \Pi_{i, j} = 1\mbox{ for all }j, \sum_{j=1}^n \Pi_{i, j} = 1\mbox{ for all }i\right\}.
\end{align*}
The papers \citep{ganassali2022sharp, wu2022settling} showed that whenever \rev{$\sigma^2 \leq n/(4+\varepsilon) \log n$}, we have $\tilde{\Pi}^\star = \Pi^\star$ w.h.p., where $\Pi^\star$ is the permutation matrix corresponding to the hidden permutation $\pi^\star$. While the above QAP guarantees to recover $\pi^\star$ w.h.p., QAPs are known to be NP-hard to even approximately solve \citep{makarychev2010maximum}. Thus, one popular approach in the literature is to consider convex relaxations of the QAP. That is, minimizing $\|AX-XB\|_F$ over some convex set $\cC_n$:
\begin{align}
X^\star\in \Argmin_{X \in \cC_n}\|AX-XB\|_F. \label{eq:convex_rel}
\end{align}
A popular choice for the convex set $\cC_n$ consists in letting $\cC_n=\cB_n$, that is the set of doubly stochastic matrices, also known as the Birkhoff polytope. By the Birkhoff-von Neumann theorem, $\cB_n$ is the convex hull of permutation matrices. Other choices for $\cC_n$ are possible; for instance \cite{araya2024graph} consider letting $\cC_n=\Delta_n$, that is the simplex in $\dR^{n^2}$, defined as
$$
\Delta_n=\left\{X\in \dR_+^{n^2}: \sum_{i,j\in [n]}X_{i, j}=n\right\}.
$$
Finally, we introduce a third option, $\cC_n=\cH_n$, where $\cH_n$ is the constrained hypercube,  defined as:
\begin{equation}\label{eq:hypercube}
\cH_n:=\left\{X\in [0,1]^{n^2}: \sum_{i,j\in [n]}X_{i, j}=n\right\}.
\end{equation}
It clearly holds that $\cB_n\subset \cH_n\subset \Delta_n$. The objective is then to establish if the solution of the convex relaxation $X^\star$ can be used to infer the hidden permutation $\Pi^\star$.
\subsection{Related Work and Main Contribution}
The related work on GOE alignment is summarized in Table~\ref{tab:algo_comparison}. Surprisingly, only a few theoretical results are known about the performance of such convex relaxations, despite strong empirical performance \cite{aflalo2015convex, dym2017ds, lyzinski2015graph}. Indeed, \cite{araya2024graph} considered the noise-free case where $\sigma=0$, for which they show that for both relaxations to $\cB_n$ and to $\Delta_n$, the unique minimizer $X^\star$ is the ground-truth permutation matrix $X^\star=\Pi^\star$ w.h.p., for which the objective function $\|AX-XB\|_F$ trivially equals zero. More recently, \cite{sushil25} showed that for relaxation to $\cB_n$, given any $\delta>0$, $X^\star$ verifies with high probability:
\begin{equation}\label{eq:birkhoff}
\begin{array}{ll}
\sigma=n^{-1-\delta}&\Rightarrow \|X^\star-\Pi^\star\|^2_F=o(n);\\
\sigma=n^{-1/2+\delta}&\Rightarrow \|X^\star-\Pi^\star\|_F^2=\Omega(n).
\end{array}
\end{equation}
In the first case, property $\|X^\star-\Pi^\star\|^2_F=o(n)$ readily guarantees that a very crude post-processing of $X^\star$, for instance letting
\begin{equation}\label{eq:post}
\forall i \in [n], \; \hat\pi(i)\in \Argmax_{j\in [n]}X^\star_{i, j},
\end{equation}
correctly recovers all but a vanishing fraction of the entries of $\pi^\star$. In this paper, we tighten the results of \cite{sushil25} by showing that the positive result in \eqref{eq:birkhoff} holds for $\sigma = \rev{\tilde{o}(n^{-1/2})}$. Combining with the negative result in \eqref{eq:birkhoff}, we now locate the threshold at which $\|X^\star-\Pi^\star\|_F^2$ goes from $o(n)$ to $\Omega(n)$\rev{, up to sub-polynomial factors}. \rev{We note that a heuristic prediction of this $n^{-1/2}$ threshold already appears in Section~1.4 of \cite{fan2023spectral}; our contribution is a rigorous proof of this transition for the actual (random) minimizer of the unregularized relaxations.} Lastly, our positive result also holds for the hypercube relaxation, i.e., we show that $\|X^\star-\Pi^\star\|_F^2 = o(n)$ for $\sigma=\rev{\tilde{o}(n^{-1/2})}$, for $X^\star$ defined in \eqref{eq:convex_rel} with $\cC_n = \cH_n$. Such a tightening needed several non-trivial innovations in the proof beyond \cite{sushil25}, which we discuss in Section~\ref{sec:discussion}.
Another related work is \cite{fan2023spectral}, which considers the following regularized convex relaxation problem, coined GRAMPA:
\begin{align}
    \underset{X : \mathbf{1}^T X \mathbf{1} = n}{\Argmin} \|AX-XB\|_F^2 + \eta \|X\|_F^2, \quad \text{for some } \eta > 0. \label{eq: grampa}
\end{align}
Note that the optimization problem above is a further relaxation of the QAP with the addition of a quadratic regularization term. In \cite{fan2023spectral}, the authors show that \eqref{eq: grampa} recovers $\Pi^\star$ whenever $\sigma = O(1/\log n)$. Such a regularized relaxation is considered in \cite{fan2023spectral} mainly due to its theoretical tractability. \rev{Specifically, \cite{fan2023spectral} drop the non-negativity constraints of the Birkhoff relaxation for tractability, and the quadratic regularization $\eta\|X\|_F^2$ is needed for good performance in this looser setting; they also consider alternative regularizations (e.g., the row-sum constraint $X\mathbf{1}=\mathbf{1}$) with better empirical performance. Empirically, the regularized relaxation \eqref{eq: grampa} trades off accuracy for speed: while it attains lower recovery accuracy than the unregularized simplex and Birkhoff relaxations, it is much faster to solve \citep{fan2023spectral, sushil25}.}
Lastly, we refer the readers to \cite{gaudio2025average, ganassali2024correlation, ding2021efficient, mao2023random} and the references within for analysis beyond the GOE setting.
\begin{table}[h]
\centering
\begin{tabular}{|l|l|l|c|}
\hline
\textbf{Paper} & \textbf{Algorithm Type} & \textbf{Algorithm Name} & \textbf{Noise $(\sigma)$} \\
\hline
\cite{ganassali2022spectral} & Top Eigenvector Alignment & EIG1     & $\Theta(n^{-7/6})$ \\
\cite{fan2023spectral}                & Regularized Relaxation & GRAMPA \eqref{eq: grampa} & $O(1/\log n)$ \\
\cite{araya2024graph}                     & Simplex Relaxation  & \eqref{eq:convex_rel} with $\cC_n = \Delta_n$     & $0$ \\
\cite{sushil25}                                & Birkhoff Relaxation  & \eqref{eq:convex_rel} with $\cC_n = \cB_n$    & $\tilde{O}(n^{-1})$ \\
\hline
Our Work                              & Hypercube/Birkhoff Relaxation  & \eqref{eq:convex_rel} with $\cC_n = \cH_n$     & $\tilde{O}(n^{-1/2})$ \\
\hline
\end{tabular}
\caption{Comparison of the spectral and optimization-based algorithms for graph alignment on the Gaussian Wigner Model: sufficient conditions for success.}
\label{tab:algo_comparison}
\end{table}
\section{Main Result} \label{sec:main_result}
We now state the main result of this section, a tightening of the positive result in \eqref{eq:birkhoff}:
\begin{theorem}\label{thm:relax}
For correlated GOE matrices as above, if noise parameter $\sigma$ satisfies
\begin{equation}
\sigma=o\left( \frac{n^{-1/2}}{\ln(n)^4}\right), \label{eq:bound_on_sigma}
\end{equation}
one then has with high probability, for both the relaxation to the Birkhoff polytope $\cB_n$ and to the constrained hypercube $\cH_n$ defined in \eqref{eq:hypercube},
$\|X^\star-\Pi^\star\|^2_F=O(\sigma n^{3/2}\ln(n)^4)=o(n).$
\end{theorem}
Before embarking on the proof, we remark that together with the negative result of \cite{sushil25} recalled in \eqref{eq:birkhoff}, Theorem \ref{thm:relax} locates the threshold at which $\|X^\star-\Pi^\star\|_F^2$ goes from $o(n)$ to $\Omega(n)$ at $\sigma=n^{-1/2}$\rev{, up to sub-polynomial factors}. The above result guarantees that the output of the post-processing $(\hat{\pi})$  defined in \eqref{eq:post} correctly recovers all but a vanishing fraction of the entries of $\pi^\star$ \rev{whenever $\sigma$ satisfies \eqref{eq:bound_on_sigma}.} Formally, define overlap $\hbox{ov}(\hat \pi,\pi^\star)$ as the fraction of indices mapped identically by both permutations, i.e.,
\begin{align*}
    \hbox{ov}(\hat \pi,\pi^\star):=\frac{1}{n}\sum_{i\in[n]}\mathbbm{1}_{\pi^\star(i)=\hat\pi(i)}.
\end{align*}
Then, $\hbox{ov}(\hat \pi,\pi^\star)=1-o(1)$ w.h.p. \rev{Indeed, fix any $i$ with $\hat\pi(i)\ne \pi^\star(i)$. Since $\hat\pi(i)\in\Argmax_{j\in[n]}X^\star_{i,j}$, we have $X^\star_{i,\hat\pi(i)}\ge X^\star_{i,\pi^\star(i)}$, so either $X^\star_{i,\pi^\star(i)}\le 1/2$ or $X^\star_{i,\hat\pi(i)}\ge X^\star_{i,\pi^\star(i)}> 1/2$. In the first case the term $(X^\star_{i,\pi^\star(i)}-1)^2\ge 1/4$; in the second the off-diagonal term $(X^\star_{i,\hat\pi(i)})^2> 1/4$ (recall $\hat\pi(i)\ne\pi^\star(i)$). In either case, row $i$ contributes at least $1/4$ to $\|X^\star-\Pi^\star\|_F^2$, so}
$$
\|X^\star-\Pi^\star\|^2_F\ge \frac{1}{4}\sum_{i\in [n]}\mathbbm{1}_{\hat\pi(i)\ne\pi^\star(i)},
$$
hence the lower bound $1-o(1)$ on the overlap when $\|X^\star-\Pi^\star\|^2_F=o(n)$.
We further remark that post-processing \eqref{eq:post} is very crude, and that a more refined post-processing could be proposed. In particular one could let $\hat\pi$ be chosen as the permutation in $\cS_n$ that solves
\begin{equation}\label{eq:postprocess}
\max_{\pi\in \cS_n}\sum_{i\in [n]}X^\star_{i \pi(i)}.
\end{equation}
This is known as the Linear Assignment Problem (LAP), or maximum weight bipartite matching problem. Due to \cite[Corollary 3]{sushil25}, we are guaranteed that the overlap of the output of the LAP $\hat{\pi}$ with $\pi^\star$ is at least $1-o(1)$ w.h.p. whenever $\sigma$ satisfies \eqref{eq:bound_on_sigma}, in light of Theorem~\ref{thm:relax}. Empirically LAP appears to achieve overlap $1-o(1)$ for $\sigma$ well above $n^{-1/2}$, but there is so far no theoretical understanding of LAP's performance in the present setup.
\rev{In fact it is not possible to tighten the Frobenius guarantees since \cite{sushil25} show that $\|X^\star - \Pi^\star\|_F^2 = \Omega(n)$ for $\sigma$ larger than $n^{-1/2}$. This gap suggests that the property governing recovery is not Frobenius proximity but a closely related one: diagonal dominance of the minimizer, i.e., $X^\star_{i\,\pi^\star(i)} > X^\star_{ij}$ for all $j \neq \pi^\star(i)$. Either property makes the post-processing above succeed; the difference is their reach, since the relaxations remain \emph{empirically} diagonally dominant---and hence keep recovering $\pi^\star$---for substantially larger $\sigma$ than the Frobenius guarantee allows. Our analysis controls the unregularized minimizer in Frobenius norm but says nothing about its diagonal dominance, and hence does not explain this empirical success for larger $\sigma$. Relatedly, \cite{fan2023spectral} \emph{do} establish diagonal dominance, though only for their \emph{regularized} minimizer (up to $\sigma = O(1/\log n)$), where it enables recovery by rounding. Establishing it for the \emph{unregularized} minimizer thus remains an important open problem.}

\rev{Our contribution is to pin down a \emph{structural} phase transition for the Birkhoff and hypercube relaxations: the point (up to sub-polynomial factors) at which the minimizer transitions from $o(n)$ to $\Omega(n)$ squared Frobenius distance from $\Pi^\star$. While this threshold does not coincide with the information-theoretic \citep{ganassali2022sharp, wu2022settling} or algorithmic frontiers \citep{fan2023spectral}, which are shown to be much larger than $\Theta(n^{-1/2})$, it is an important changepoint in the behavior of the unregularized Birkhoff relaxation, above which different proof techniques appear to be required to establish any form of correctness guarantee. Since understanding the (most natural) Birkhoff relaxation is known to be difficult---\cite{fan2023spectral} describe it as ``a challenging task yet to be accomplished''---we believe that precisely characterizing its different operating regimes is of independent interest. Moreover, our techniques (upcoming Lemmas~\ref{lem:relax} and~\ref{lemma:2}) are not regime-specific and may be useful at larger noise levels.}
\section{Proof of Theorem~\ref{thm:relax}}
Thanks to invariance of Frobenius norm $\|\cdot\|_F$ and of both $\cB_n$ and $\cH_n$ by left- or right-multiplication by a permutation matrix, we can assume without loss of generality that $\pi^\star$ is the identity. We thus aim to prove
\begin{equation}
\|X^\star-I\|_F^2=o(n), \label{eq:close_to_I}
\end{equation}
where $X^\star$ minimizes
\begin{equation}\label{eq:obj}
\|AX -X(A+\sigma Z)\|_F
\end{equation}
over  $X\in \cC_n$ where $\cC_n$  could be either $\cH_n$ or $\cB_n$. To ensure \eqref{eq:close_to_I}, it is sufficient to show that $\Tr(X^\star)$ is at least $n(1-o(1))$ w.h.p. In particular, as $X^\star \in \cH_n$, we have
\begin{align}
    \|X^\star-I\|_F^2 = \sum_{i,j \in [n]: i \neq j} (X^\star_{i,j})^2 + \sum_{i=1}^n (X^\star_{i, i}-1)^2 = \|X^\star\|_F^2 + n - 2\Tr(X^\star) \leq 2 \left(n - \Tr(X^\star)\right), \label{eq:trace_bound}
\end{align}
where the last inequality follows as $X^\star \in \cH_n$, which implies $\|X^\star\|_F \leq \sqrt{n}$. The above inequality ascertains that $\Tr(X^\star) = n(1-o(1))$ implies $\|X^\star-I\|_F^2 = o(n)$. So, we focus on establishing the former. Another key ingredient is the following upper bound on the objective \eqref{eq:obj}:
\begin{align}
    \|AX^\star-X^\star B\|_F \leq \|A-B\|_F = \sigma \|Z\|_F \overset{(*)}{=} O(\sigma \sqrt{n}) 
    , \label{eq:bound_on_obj}
\end{align}
where $(*)$ holds as $\|Z\|_F \leq c \sqrt{n}$ w.p. $1-o(1)$ for $c>0$ large enough.
We now present two different proofs for the assertion that \eqref{eq:bound_on_obj} implies $\Tr(X^\star)$ is close to $n$.
\subsection{Direct (Primal) Approach}
In this approach, we establish the following lower bound on the objective of \eqref{eq:obj}:
$$
\forall X\in \cC_n,\; \|AX-XB\|_F\ge \frac{1}{O(n\log^4(n))}(n- \Tr(X)),
$$
which then combined with \eqref{eq:bound_on_obj} and \eqref{eq:bound_on_sigma} implies $\Tr(X^\star) \geq n(1-o(1))$. We present the complete proof below.
\begin{proof}[Proof of Theorem~\ref{thm:relax}] The proof is mainly divided into the following four parts.

\noindent \underline{\emph{Isolating dependence on $B$:}} The objective \eqref{eq:obj} admits the lower bounds
$$
\|AX -X(A+\sigma Z)\|_F\ge \|AX-XA\|_F-\sigma\|XZ\|_F\ge \|AX-XA\|_F-\sigma \|X\|_F \|Z\|_2,
$$
where $\|Z\|_2$ denotes the operator norm of matrix $Z$, and we used the classical inequality $\|XZ\|_F\le \|Z\|_2\|X\|_F$, valid for any two matrices $X,Z$.  By classical results on spectra of random matrices \cite{anderson}, with high probability matrix $Z$ verifies $\|Z\|_2\le 2+o(1)=O(1)$. Also, any matrix $X$ in $\cC_n$ verifies $\|X\|_F\le \sqrt{n}$, so that
\begin{equation}\label{eq:tmp_tmp}
\|AX -X(A+\sigma Z)\|_F\ge \|AX-XA\|_F - 3 \sigma \sqrt{n}.
\end{equation}
\underline{\emph{Lower bound by a linear function:}} Let $\lambda_1\le \cdots\le \lambda_n$ denote the spectrum of $A$, and $u_1,\ldots,u_n$ an associated orthonormal basis of eigenvectors of $A$. The matrix $(u_1|\cdots|u_n)$ is uniformly distributed over the orthogonal group $O_n$. \rev{The goal of this step is to lower bound $\|AX-XA\|_F^2$ in terms of $n-\Tr(X)$, so that the objective is large whenever $X$ is far from $I$. The natural approach---bounding by the smallest eigenvalue of $A\otimes I - I\otimes A$, namely $\min_{i,j}(\lambda_i-\lambda_j)^2$---fails, since $\lambda_i-\lambda_j$ can be very small when $|i-j|$ is small. We therefore project $X-I$ onto the eigenvector pairs $u_iu_j^\top$ with $|i-j|\ge \epsilon n$, so that the minimum eigenvalue gap restricted to this projected subspace can be controlled. More formally,} let $\epsilon>0$ be chosen, possibly dependent on $n$. Let
\begin{equation}\label{eq:def_c_epsilon}
C_\epsilon:=\min_{i,j\in[n]:|i-j|\ge \epsilon n} (\lambda_i-\lambda_j)^2.
\end{equation}
Denote by $P_{\epsilon}$ the orthonormal projection defined by
\begin{equation}
P_{\epsilon}(M):=\sum_{i,j:|i-j|<\epsilon n}\langle u_i u_j^\top,M\rangle u_i u_j^\top.
\end{equation}
\rev{Write then
\begin{equation}\label{eq:array_bound}
\begin{array}{ll}
\|AX-XA\|_F^2 &=\|A(X-I)-(X-I)A\|_F^2\\
&=\sum_{i,j\in [n]}(\lambda_i-\lambda_j)^2\langle u_i u_j^\top, X-I\rangle^2\\
&\ge C_\epsilon \|(I-P_\epsilon)(X-I)\|^2_F\\
&\overset{(a)}{\ge} C_\epsilon \frac{\langle (I-P_\epsilon)(J),X-I\rangle^2}{\|J\|_F^2}\\
&\overset{(b)}{=}C_\epsilon n^{-2}\langle (I-P_\epsilon)(J),X-I\rangle^2 \\
&\overset{(c)}{=} C_\epsilon n^{-2}\langle P_\epsilon(J),I-X\rangle^2,
\end{array}
\end{equation}
where we introduced the all-ones matrix $J$ and used Cauchy-Schwarz inequality in $(a)$, used $\|J\|_F=n$ in $(b)$, and used $\langle J, X \rangle = \langle J, I \rangle = n$ in $(c)$.}
\rev{For the lower bound \eqref{eq:array_bound} to reflect the trace deficiency $n-\Tr(X)$, we need $Y:=P_\epsilon(J)$ to be close to $I$. This is governed by $\epsilon$, since $Y$ interpolates between $J$ for large $\epsilon$ and $I$ for small $\epsilon$. To see this, write $J=ee^\top=\sum_{i,j\in[n]}(u_i^\top e)(u_j^\top e)\,u_iu_j^\top$. At $\epsilon=1$, $Y$ sums over all pairs and equals the matrix $J$; at $\epsilon=1/n$, $Y$ retains only the diagonal terms $\sum_i(u_i^\top e)^2\,u_iu_i^\top$, which can be shown to be component-wise close to $I=\sum_i u_iu_i^\top$ since the eigenvectors $u_i$ are delocalized (being uniform on $\cS^{n-1}$). We thus want $\epsilon$ small enough that $Y\approx I$, yet large enough that the gap $C_\epsilon=\min_{|i-j|\ge\epsilon n}(\lambda_i-\lambda_j)^2$ does not vanish; the next step balances this trade-off.}

\noindent \underline{\emph{Analysis at extreme points of $\cC_n$:}} The extremal elements of $\cC_n$ are its elements whose entries belong to $\{0,1\}$. For some $\cC \subset [n]^2$ of size $n$, we denote by $X_\cC$ the extremal element of $\cC_n$ whose support is $\cC$ (if it exists).
In addition, we denote by $\cC_1$ the set of pairs $(i,j)$ in $\cC$ with $i\ne j$, and by $\cC_2$ the set of indices $k$ such that $(k,k)$ is not in $\cC$.
We can write
$$
X_\cC=I+ \sum_{(i,j)\in \cC_1}e_i e_j^\top-\sum_{k\in \cC_2}e_k e_k^\top.
$$
Note that
$$
|\cC_1|=|\cC_2|=n-\Tr X_\cC.
$$
\rev{Write then
\begin{align*}
\left\langle P_\epsilon(J),I-X_\cC\right\rangle&=
\displaystyle
\left\langle P_\epsilon(J),\sum_{k\in \cC_2}e_k e_k^\top-\sum_{(i,j)\in \cC_1}e_i e_j^\top\right\rangle\\
&=
\displaystyle \sum_{k\in \cC_2}\left\langle P_\epsilon (J), e_k e_k^\top\right\rangle
-\sum_{(i,j)\in \cC_1}\left\langle P_\epsilon (J), e_i e_j^\top\right\rangle
\end{align*}}
Let now $Y:=P_\epsilon (J)$.
The previous display is then written as
\begin{equation}\label{eq:lower_bound_zz}
\langle P_\epsilon(J),I-X_\cC\rangle= \sum_{k\in \cC_2}Y_{kk}-\sum_{(i,j)\in \cC_1} Y_{ij}.
\end{equation}
A key step to conclude our argument is the following
\begin{lemma}\label{lem:relax}
Choose $\epsilon=\kappa/\ln^4(n)$, where $\kappa>0$ is a sufficiently small constant, of order $\Omega(1)$. Then with high probability one has:
\begin{equation}\label{eq:Zii}
\forall i\in [n],\; Y_{ii}=1+o(1),
\end{equation}
and
\begin{equation}\label{eq:Zij}
\forall i\ne j\in [n],\; |Y_{ij}|\le \frac{1}{2}\cdot
\end{equation}
\end{lemma}
The proof of Lemma~\ref{lem:relax} is presented in Section~\ref{sec:lemrelax}.
\rev{Intuitively, the choice $\epsilon=\kappa/\ln^4(n)$ is essentially the largest band for which $Y=P_\epsilon(J)$ remains component-wise close to $I$. In particular, recall $Y$ is close to $I$ for small $\epsilon$ and close to $J$ for large $\epsilon$: Lemma~\ref{lem:relax} certifies that $\epsilon=\kappa/\ln^4(n) = o(1)$ is still on the $Y\approx I$ side. Essentially, we take $\epsilon$ as large as possible, while ensuring $Y\approx I$, because a larger band yields a larger eigenvalue gap $C_\epsilon$; we show later (Lemma~\ref{lemma:2}) that this choice gives $C_\epsilon=\tilde\Omega(1)$, keeping the prefactor in the resulting lower bound from vanishing.}
Plugging the above lemma's result in \eqref{eq:lower_bound_zz} we get
\begin{align*}
\langle P_\epsilon(J),I-X_\cC\rangle\ge(1-o(1))|\cC_2|-\frac{1}{2}|\cC_1|\ge \frac{1}{3}|\cC_2|=\frac{1}{3}(n-\Tr(X_\cC)).
\end{align*}
\underline{\emph{Concluding:}} For each $X\in \cC_{n}$, writing $X$ as a convex combination of extremal points of $\cC_n$ of the form $X_{\cC}$ as defined above, one readily obtains
$$
\forall X\in \cC_{n}, \; \langle P_\epsilon(J),I-X\rangle\ge \frac{1}{3}(n-\Tr(X)).
$$
Combined with \eqref{eq:array_bound}, this yields
$$
\forall X\in \cC_n,\; \|AX-XA\|_F\ge \frac{\sqrt{C_\epsilon}}{n}\frac{1}{3}(n- \Tr(X)).
$$
This holds in particular for the optimum $X^\star$. Its optimality ensures that it achieves a lower objective value than the identity matrix $I$, for which $\|A\times I -I(A+\sigma Z)\|_F=\sigma\|Z\|_F=O(\sigma \sqrt{n})$, so that by \eqref{eq:tmp_tmp} and the previous display:
$$
O(\sigma \sqrt{n})\ge \|AX^\star -X^\star(A+\sigma Z)\|_F\ge \frac{\sqrt{C_\epsilon}}{3}\left[1-n^{-1}\Tr X^\star\right] - 3 \sigma\sqrt{n}.
$$
Thus
$$
\Tr X^\star\ge n-O(\sigma n^{3/2}/\sqrt{C_\epsilon}).
$$
The following lemma will then allow us to conclude:
\begin{lemma}\label{lemma:2}
Assume that for some fixed constant $\delta>0$ one has
\begin{equation}\label{eq:constraint_epsilon}
\epsilon\ge n^{\delta-2/3}.
\end{equation}
Then with high probability
\begin{equation}
C_\epsilon=\Omega(\epsilon^2),
\end{equation}
where $C_\epsilon$ is defined in \eqref{eq:def_c_epsilon}.
\end{lemma}
Our choice $\epsilon=\Omega(1)/\ln^4(n)$ clearly satisfies Condition \eqref{eq:constraint_epsilon}. Lemma \ref{lemma:2} thus ensures that $C_\epsilon=\Omega(\epsilon^2)$. The proof of Lemma~\ref{lemma:2} is presented in Appendix~\ref{app:lemma2}. \rev{The intuition behind Lemma~\ref{lemma:2} is that GOE eigenvalues have typical spacing $\lambda_{i+1}-\lambda_i\sim 1/n$, so two eigenvalues whose indices differ by at least $\epsilon n$ are separated by roughly $\epsilon$; squaring gives $C_\epsilon=\Omega(\epsilon^2)$.}
\end{proof}
\subsection{Dual Approach}
In this section, we prove the following result:
\begin{proposition}\label{prop:relax}
For correlated GOE matrices, if the noise parameter $\sigma$ satisfies
\begin{equation}
\sigma=o\left( \frac{n^{-1/2}}{\ln(n)^{5.5}}\right), \label{eq:relaxed_bound_on_sigma}
\end{equation}
one then has with high probability, for relaxation to the Birkhoff polytope $\cB_n$, constrained hypercube $\cH_n$, and simplex $\Delta_n$:
\begin{equation}
\Tr(X^\star) \geq n(1-o(1)).
\end{equation}
\end{proposition}
In light of \eqref{eq:trace_bound}, the above proposition implies the conclusion of Theorem~\ref{thm:relax} under Condition \eqref{eq:relaxed_bound_on_sigma}, a slight strengthening of Condition \eqref{eq:bound_on_sigma}  on $\sigma$. Note that the above proposition also ensures that $\Tr(X^\star) \geq n(1-o(1))$, even for the simplex relaxation, i.e., $\cC_n = \Delta_n$. However,
\rev{the reduction from $\Tr(X^\star)\ge n-o(n)$ to $\|X^\star-I\|_F^2=o(n)$ via \eqref{eq:trace_bound} requires $\|X^\star\|_F\le \sqrt{n}$, which may fail for $X^\star\in\Delta_n$. Proposition~\ref{prop:relax} still yields $\Tr(X^\star)\ge n-o(n)$ for the simplex; whether this can be strengthened to a Frobenius-norm guarantee remains open.}
\begin{proof}[Proof of Proposition~\ref{prop:relax}]
As before, let $A = \sum_{i=1}^n \lambda_i u_iu_i^T$, where $\{\lambda_i\}_{i=1}^n$ are the eigenvalues and $\{u_i\}_{i=1}^n$ are the corresponding orthonormal eigenvectors of $A$. The proof is mainly divided into the following four parts.

\noindent \underline{\emph{Trace minimization problem:}} Our objective is to lower bound $\Tr(X^\star)$ for $X^\star \in \Delta_n$ satisfying \eqref{eq:bound_on_obj}, which motivates the following optimization problem:
\begin{align}
    T^\star = \min_{X \in \bbR^{n \times n}} \text{Tr}(X) \quad \text{subject to} \sum_{i,j=1}^n X_{ij} = n, X \geq 0, \|AX-XB\|_F^2 \leq c^2\sigma^2 n. \label{eq:trace_min}
\end{align}
Note that the objective is a linear function, and the constraints define a convex set. Thus, the above optimization problem is convex. \rev{The norm constraint $\|AX-XB\|_F^2\le c^2\sigma^2 n$ encodes the (near-) optimality of $X^\star$ via \eqref{eq:bound_on_obj}, and is essential: without it, one could trivially construct dual solutions with objective value $0$.}

\noindent \underline{\emph{The Dual Problem:}} Introduce the dual variables $R \in \bbR^{n \times n}$ and $\tilde{\mu}, \mu \in \bbR$ corresponding to the non-negativity, the norm constraint, and the total sum constraint, respectively. Thus, by weak duality, we have
\begin{align*}
   T^\star \geq{}& \max_{R \geq 0, \mu, \tilde{\mu}} \min_{X} \text{Tr}(X) - \langle R, X \rangle - \mu \langle J, X \rangle + \tilde{\mu}\|AX-XB\|_F^2 - \tilde{\mu} c^2\sigma^2 n + \mu n \\
   ={}&\max_{R \geq 0, \mu, \tilde{\mu}} \min_{X}  \tilde{\mu}\|AX-XB\|_F^2 - \langle R+\mu J - I, X \rangle - \tilde{\mu} c^2\sigma^2 n + \mu n \\
   \overset{(a)}{=}{}& \max_{R \geq 0, \mu, \tilde{\mu}, M}   \tilde{\mu}\|AM-MB\|_F^2 - \langle R+\mu J - I, M \rangle - \tilde{\mu} c^2\sigma^2 n + \mu n \\
   &\text{subject to} \quad 2\tilde{\mu}(A^2M-2AMB+MB^2) = R+\mu J - I \\
   \overset{(b)}{=}{}& \max_{R \geq 0, \mu, \tilde{\mu}, M}   -\tilde{\mu}\|AM-MB\|_F^2 - \tilde{\mu} c^2\sigma^2 n + \mu n \\
   &\text{subject to} \quad 2\tilde{\mu}(A^2M-2AMB+MB^2) = R+\mu J - I \\
   \overset{(c)}{=}{}& \max_{R \geq 0, \mu, \tilde{\mu}, \tilde{M}, M}   -\tilde{\mu}\|\tilde{M}\|_F^2  - \tilde{\mu} c^2\sigma^2 n + \mu n \quad \\
   &\text{subject to} \quad 2\tilde{\mu}(A\tilde{M}-\tilde{M}B) = R+\mu J - I, \ \tilde{M}=AM-MB \numberthis \label{eq:constraints_dual}.
\end{align*}
Above, $(a)$ follows from the first order optimality condition by noting that the inner minimization is an unconstrained convex problem over $X$; the constraint on $M$ corresponds to the gradient of the minimized quantity being $0$. Next, $(b)$ holds by multiplying the constraint on both sides by $M$, which implies
\begin{align*}
    \langle R+\mu J - I, M \rangle = 2\tilde{\mu} \langle A^2M-2AMB+MB^2, M \rangle = 2\tilde{\mu}\|AM-MB\|_F^2.
\end{align*}
Lastly, $(c)$ follows by introducing the auxiliary variable $\tilde{M} = AM-MB$.

\noindent \underline{\emph{Dual Certificate:}} Now, we provide a dual feasible solution $(R, \mu, \tilde{\mu}, M, \tilde{M})$ which provides a lower bound on $T^\star$. Denote $e$ the all-ones vector. Fix $\epsilon = \kappa/(\log n)^4$ for some large enough $\kappa > 0$, set $\tilde{\mu} = n (\log n)^{10.5}$, and
\begin{align*}
    \tilde{M} = \frac{1}{\tilde{\mu}}\sum_{|i-j| \geq \epsilon n} \frac{(e^T u_i) (e^T u_j)}{\lambda_i-\lambda_j} u_i u_j^T, \ Y = P_\epsilon(J) = \sum_{|i-j| < \epsilon n} (e^T u_i) (e^T u_j) u_i u_j^T,
\end{align*}
Moreover, we also set $\mu=1-\xi_1 - \xi_2$ with $\xi_1 = 2\tilde{\mu}\sigma \max_{i, j \in [n]} (\tilde{M}Z)_{i, j}$ and $\xi_2 = \max_{i, j \in [n]} (2Y-I-J)_{i, j}$, and
\begin{align*}
    R = \left(1+\xi_1 + \xi_2\right)J + I - 2Y - 2\tilde{\mu}\sigma \tilde{M} Z. \numberthis \label{eq:dualR}
\end{align*}
Next, note that $A \otimes I - I \otimes B$ is non-singular with probability 1 for $\sigma > 0$ \cite[Lemma 4]{araya2024graph}. Thus, we set $M$ to be the unique solution of $\tilde{M} = AM-MB$. The constraints are feasible by construction. The feasibility of equality constraint is verified as follows:
\begin{align*}
    \tilde{\mu}(A\tilde{M}-\tilde{M}B) &= \tilde{\mu}(A\tilde{M}-\tilde{M}A)- \tilde{\mu} \sigma \tilde{M} Z \\
    &=\sum_{|i-j| \geq \epsilon n} \frac{(e^T u_i) (e^T u_j)}{\lambda_i-\lambda_j} (Au_i u_j^T-u_i u_j^T A)-\tilde{\mu} \sigma \tilde{M} Z \\
    &= \sum_{|i-j| \geq \epsilon n}(e^T u_i) (e^T u_j) u_i u_j^T-\sigma \tilde{\mu} \tilde{M} Z = J - Y -  \tilde{\mu} \sigma \tilde{M}Z = (R+\mu J-I)/2.
\end{align*}
Also, by the definition of $\xi_1, \xi_2$, we have $R = \left(\xi_1 J - 2\tilde{\mu} \sigma \tilde{M}Z\right) + \left(\xi_2 J - (2Y-I-J)\right) \geq 0$.
\noindent \underline{\emph{Dual Objective Value:}} We first upper bound $\|\tilde{\mu}\tilde{M}\|_F$ using eigenvalue separation established in Lemma~\ref{lemma:2}. Indeed, with high probability,
\begin{align}
    \|\tilde{\mu}\tilde{M}\|_F^2 = \sum_{|i-j| \geq \epsilon n } \frac{(e^Tu_i)^2 (e^T u_j)^2}{(\lambda_i-\lambda_j)^2} \leq \frac{1}{C_\epsilon} \sum_{|i-j| \geq \epsilon n } (e^Tu_i)^2 (e^T u_j)^2 \overset{(*)}{\leq} O\left(n^2 (\log n)^{10}\right), \label{eq: norm_M}
\end{align}
where $(*)$ holds as, for any $i\leq n$, for any $c>1$, it holds with probability $1-n^{-c}$ for large enough $n$ that
\begin{align*}
    e^T u_i = \frac{e^T z}{\|z\|_2} \leq \frac{2 e^Tz}{\sqrt{n}} \leq O(\sqrt{\log n}),
\end{align*}
where $z \sim N(0, I_n)$. The second inequality holds from \cite[Lemma 15]{fan2023spectral}, and the last one from \cite[Lemma 13]{fan2023spectral}. By union bound, the same holds true simultaneously for all $i$ with probability at least $1-n^{-(c-1)}$. Moreover, $C_\epsilon = \Omega(\epsilon^2)$ by Lemma~\ref{lemma:2}, where, we set $\epsilon = \kappa/(\log n)^4$.
Next, we upper bound $\xi_1$. As $A$ and $Z$ are independent, $\tilde{M}$ is independent of $Z$. Thus, by \cite[Lemma 13]{fan2023spectral}, for any $i,j$, for any $c > 2$, with probability at least $1 - n^{-c}$, we have
\begin{align}
    |(\tilde{M}Z)_{i, j}| = \bigg|\sum_{l=1}^n \tilde{M}_{i, l} Z_{l, j}\bigg| \leq O\left(\sqrt{\frac{\log n}{n}\sum_{l=1}^n \tilde{M}_{i, l}^2} \right) \leq O\left(\sqrt{\frac{\log n}{n}} \|\tilde{M}\|_F\right). \label{eq: mz_bound}
\end{align}
This inequality simultaneously holds for all pairs $(i,j)$, with probability $1-n^{-(c-2)}$, by union bound.
Thus, we get
\begin{align}
    \xi_1 = \max_{i, j \in [n]}|2\tilde{\mu} \sigma(\tilde{M}Z)_{ij}| = O\left(\sigma\sqrt{\frac{\log n}{n}} \|\tilde{\mu}\tilde{M}\|_F\right) \overset{\eqref{eq: norm_M}}{=} O\left(\sigma\sqrt{n} (\log n)^{5.5}\right) = o(1), \label{eq:xi1}
\end{align}
where the last equality holds by \eqref{eq:relaxed_bound_on_sigma}. Now, we upper bound $\xi_2$. By Lemma~\ref{lem:relax}, we readily obtain $(2Y-I-J)_{i, j} \leq 0$ for $i \neq j \in [n]$ and $(2Y-I-J)_{i, i} = o(1)$ w.h.p. Thus, we have
\begin{align}
    \xi_2 = \max_{i, j \in [n]} (2Y-I-J)_{i, j} = o(1), \label{eq:xi2}
\end{align}
w.h.p. Combining \eqref{eq: norm_M}, \eqref{eq:xi1}, and \eqref{eq:xi2}, we get
\begin{align*}
    T^\star &\geq -\tilde{\mu}\|\tilde{M}\|_F^2  - \tilde{\mu} c^2\sigma^2 n + \mu n = -\tilde{\mu}\|\tilde{M}\|_F^2  - \tilde{\mu} c^2\sigma^2 n + \left(1-\xi_1-\xi_2\right) n\\
    &\geq -O\left(\frac{n^2 (\log n)^{10}}{\tilde{\mu}}\right)- \tilde{\mu} c^2 \sigma^2 n + n\left(1 - o(1)\right) = n - o(n),
\end{align*}
where the last equality holds due to our choices of $\tilde{\mu} = n(\log n)^{10.5}$ and \eqref{eq:relaxed_bound_on_sigma}. Thus, we obtain $T^\star \geq n -o(n)$, which completes the proof.
\end{proof}
\subsection{Discussion on the two approaches} \label{sec:discussion}
The primal and the dual approaches both provide a lower bound on $\Tr(X^\star)$ by using the upper bound on the objective function given in \eqref{eq:bound_on_obj}.
The primal approach lower bounds the trace for any feasible $X \in \cH_n$ of \eqref{eq:trace_min}, while the dual approach constructs a dual feasible solution, whose objective function value lower bounds the trace via weak duality. \rev{Although the two approaches look different, they run into the \emph{same} obstacle and resolve it with the \emph{same} technical hammer, which is precisely why Lemmas~\ref{lem:relax} and~\ref{lemma:2} are central to both. The obstacle is that the map $X\mapsto AX-XA$ has eigenvalues $\{\lambda_i-\lambda_j\}_{i,j\in[n]}$, which vanish at $i=j$ and are tiny when $|i-j|$ is small; the technical hammer is to restrict attention to the well-separated pairs $|i-j|\ge\epsilon n$, on which Lemma~\ref{lemma:2} controls the eigengap $C_\epsilon$ and Lemma~\ref{lem:relax} controls the relevant projection of $J$. We make this parallel precise below.}
\rev{In the primal approach the obstacle appears directly: the naive bound $\|AX-XA\|_F = \|A(X-I)-(X-I)A\|_F\ge \lambda_{\min}(A\otimes I-I\otimes A)\,\|X-I\|_F$ is vacuous, since the eigenvalues $\{\lambda_i-\lambda_j\}$ vanish at $i=j$ and can be as small as $n^{-3/2}$ for nearby indices \citep{smallgapsofgoe}. Projecting $X-I$ onto the band $|i-j|\ge\epsilon n$ excises these small gaps: Lemma~\ref{lemma:2} ensures $C_\epsilon=\tilde\Omega(1)$, while Lemma~\ref{lem:relax} gives $P_\epsilon(J)\approx I$, so that the projected objective still captures $n-\Tr(X)$ via Cauchy--Schwarz, as in \eqref{eq:array_bound}.}
\rev{The dual approach meets the identical obstacle, only inverted. Constructing a dual certificate for \eqref{eq:trace_min} requires the primal solution to be close to $I$, which by complementary slackness suggests the choice $R=J-I$ (exactly the certificate used in \cite{sushil25}). The dual constraint \eqref{eq:constraints_dual} then forces $\tilde{\mu}\tilde{M}$ to satisfy $\tilde{\mu}(A\tilde{M}-\tilde{M}A)\approx J-I$, i.e., $\tilde{\mu}\tilde{M}$ is obtained by inverting the same map $X\mapsto AX-XA$ on $J$:}
$$
\rev{\tilde{\mu}\tilde{M}\approx \sum_{i\ne j}\frac{(e^\top u_i)(e^\top u_j)}{\lambda_i-\lambda_j}\,u_iu_j^\top .}
$$
\rev{The very same small denominators $\lambda_i-\lambda_j$ that weakened the primal bound now \emph{blow up} $\|\tilde{\mu}\tilde{M}\|_F$, producing a small dual objective and hence a weak bound on $\Tr(X^\star)$; this is what restricts \cite{sushil25} to $\sigma=o(n^{-1})$. The remedy is the same as in the primal proof: we set $R$ to approximately the projection of $J-I$ onto $\{u_iu_j^\top\}_{|i-j|\ge\epsilon n}$ (the remaining terms in \eqref{eq:dualR} are lower order), which excises the small-gap pairs and yields $\tilde{\mu}\tilde{M}\approx \sum_{|i-j|\ge\epsilon n}\frac{(e^\top u_i)(e^\top u_j)}{\lambda_i-\lambda_j}u_iu_j^\top$. Now Lemma~\ref{lemma:2} bounds $\|\tilde{\mu}\tilde{M}\|_F$ through the same eigengap $C_\epsilon$, while Lemma~\ref{lem:relax} controls the error incurred by projecting $J-I$ off the band, ensuring $R$ stays close to $J-I$ and feasible. In short, the two proofs are mirror images: the band $|i-j|\ge\epsilon n$ and the two lemmas play identical roles in each, one controlling the eigengap $C_\epsilon$ (Lemma~\ref{lemma:2}) and the other a projection of $J$ (Lemma~\ref{lem:relax}).}
Establishing these lemmas (especially Lemma~\ref{lem:relax}) requires considerable technical overhead beyond \cite{sushil25}, highlighting our methodological contributions. The advantage of the dual approach compared to the direct approach is that it provides a tight lower bound on $\Tr(X^\star)$ even for the simplex relaxation $(\cC_n = \Delta_n)$. While this trace bound cannot directly be translated to guarantees on $\|X^\star-I\|_F$, it still provides partial results and progress towards that goal, and could serve as a good starting point for future work on simplex relaxation.
\section{Proof of Lemma \ref{lem:relax}.}\label{sec:lemrelax}
We shall consider the case of $Y_{i, j}$ for $i \neq j$ and $i = j$ separately. The details for the case of $i=j$ are deferred to Appendix~\ref{sec:iequalj} as it is easier than the case of $i \neq j$. Here we focus on bounding $Y_{i, j}$ for $i \neq j$. Denote by $e$ the all-ones vector. Write, recalling that $U=(u_1|\cdots |u_n)$, and denoting $V:=U^\top$:
\begin{align*}
Y_{i, j}&=\displaystyle \sum_{a,b:|b-a|<\epsilon n}\langle u_a u_b^\top, e e^\top\rangle\langle u_au_b^\top,e_i e_j^\top\rangle
= \sum_{a,b:|b-a|<\epsilon n} (u_a^\top e) (u_b^\top e) (u_b^\top e_i)(u_a^\top e_j)\\
&=
\sum_{a,b:|b-a|<\epsilon n}(e_a^\top Ve)(e_b^\top V e)(e_b^\top Ve_i)(e_a^\top V e_j).
\end{align*}
Now, $V$ is uniformly distributed over the orthogonal group $O_n$. Thus, $Ve_i$ is uniformly distributed over $\cS^{n-1}$, $V e_j$ is uniformly distributed over the intersection of $\cS^{n-1}$ with the orthogonal of $V e_i$, and $(n-2)^{-1/2}[Ve - V e_i -V e_j]$ is uniformly distributed over the intersection of $\cS^{n-1}$ with the orthogonal of $\hbox{vect}\{Ve_i,Ve_j\}$. To control the distribution of $Y_{i, j}$ we use the following construction. With a slight abuse of notation, take $X$, $X'$, $X''$ three i.i.d. Gaussian random vectors with distribution $\cN(0,I)$. Let then
$$
v:=\frac{1}{\alpha} X,\quad
v':=\frac{1}{\beta}(X'-r v),\quad
v'':=\frac{1}{\gamma}(X'' -s v -t v')
$$
denote the Gram-Schmidt orthonormalization of $(X,X',X'')$, where
\begin{align*}
\alpha:=\|X\|, \;
r:= v^\top X', \;
\beta:=\|X'-rv\|, \;
s:=v^\top X'', \;
t:=v'^\top X'', \;
\gamma:=\|X''-sv-tv'\|.
\end{align*}
Then $Y$ has the same distribution as
$$
Y:=\sum_{a,b:|b-a|<\epsilon n}v_b v'_a \left(v_a+v'_a +\sqrt{n-2}v''_a\right)\left(v_b+v'_b+\sqrt{n-2}v''_b\right).
$$
For two independent standard Gaussian random variables $Z_1$, $Z_2$ and $x\in\dR$, we have
$$
    \dE\left(e^{x Z_1^2}\right)=(1-2x)_+^{-1/2}, \;
    \dE\left(e^{x Z_1 Z_2}\right)=(1-x^2)_+^{-1/2}.
$$
Let $\bar{Z}_1=(Z_1(1),\ldots, Z_1(n))$, $\bar{Z}_2=(Z_2(1),\ldots,Z_2(n))$ be two independent vectors both distributed according to $\cN(0,I)$.
By Chernoff's argument, we have the existence of two non-negative convex functions $g$, $h$, where $g$ achieves its minimum at $1$, with $g(1)=g'(1)=0$ and $g''(1)>0$, and likewise $h$ achieves its minimum at $0$ with $h(0)=h'(0)=0$, $h''(0)>0$, such that for all $x>0$,
$$
\dP\left(\bigg|\sum_{i=1}^n Z_1(i)^2-n\bigg|\ge xn\right)\le  2 e^{-n g(1+x)}
 \hbox{ and }
\dP\left(\bigg|\sum_{i=1}^n Z_1(i)Z_2(i)\bigg|\ge x n\right)\le 2 e^{-n h(x)}.
$$
Thus for any constant $c>0$, for large enough constant $C>0$ (for instance, we can set $C=3c\max(1/g''(1),1/h''(0))$), taking $x=C\sqrt{\ln(n)/n}$ we obtain
$$
\dP\left(\bigg|\sum_{i=1}^n Z_1(i)^2-n\bigg|\ge C\sqrt{n\ln(n)}\right)\le n^{-c}
\hbox{ and }
\dP\left(\bigg|\sum_{i=1}^n Z_1(i)Z_2(i)\bigg|\ge C\sqrt{n\ln(n)}\right)\le n^{-c}.
$$
Thus with probability $1-O(n^{-c})$, one has:
$$
\begin{array}{ll}
\alpha&=\sqrt{n}(1+O(\sqrt{\ln(n)/n}))=\sqrt{n}+O(\sqrt{\ln(n)}),\\
r&=\alpha^{-1}X^\top X'=O(\alpha^{-1}\sqrt{n\ln(n)})=O(\sqrt{\ln(n)}),\\
\beta&=\sqrt{n+O(\sqrt{n\ln(n)})+O(r/\alpha)\sqrt{n\ln(n)}+(r/\alpha)^2(n+O(\sqrt{n\ln(n)}))}\\
&=\sqrt{n}+O(\sqrt{\ln(n)}),\\
s,t&=O(\sqrt{\ln(n)}),\;
\gamma=\sqrt{n}+O(\sqrt{\ln(n)}).
\end{array}
$$
Write then, denoting $r':=r/\alpha$, $s'=s/\alpha$, $t'=t/\beta$:
\begin{align*}
\alpha \beta Y={}&\sum_{|a-b|<\epsilon n}X_a(X'_b-r' X_b)\times \\
&\times\left[\left(\frac{1}{\alpha}-\frac{r'}{\beta}-\frac{\sqrt{n-2}s'}{\gamma}+\frac{r't'\sqrt{n-2}}{\gamma}\right)X_a+\left(\frac{1}{\beta}-\frac{\sqrt{n-2}t'}{\gamma}\right)X'_a+\frac{\sqrt{n-2}}{\gamma}X''_a\right]\\
&\times\left[\left(\frac{1}{\alpha}-\frac{r'}{\beta}-\frac{\sqrt{n-2}s'}{\gamma}+\frac{r't'\sqrt{n-2}}{\gamma}\right)X_b+\left(\frac{1}{\beta}-\frac{\sqrt{n-2}t'}{\gamma}\right)X'_b+\frac{\sqrt{n-2}}{\gamma}X''_b\right]\\
={}& \sum_{|a-b|<\epsilon n}X_a(X'_b+\theta X_b)\left[\theta X_a+\theta X'_a+(1+o(1))X''_a\right]\left[\theta X_b+\theta X'_b+(1+o(1))X''_b\right],
\end{align*}
where in the last expression, $\theta$ is a short-hand notation for $O\left(\sqrt{\ln n/n}\right)$. This yields:
\begin{align*}
\alpha \beta Y=&(1+o(1))\sum_{|a-b|<\epsilon n}X_aX'_b X''_a X''_b+\theta^3\sum_{|a-b|<\epsilon n}X_a X_b(X_a+X'_a)(X_b+X'_b)\\
&+\theta\sum_{|a-b|<\epsilon n}X_a[X'_bX''_a(X_b+X'_b)+X'_b(X_a+X'_a)X''_b+X_b X''_a X''_b]\numberthis \label{eq:y_as_function_of_x}\\
&+\theta^2\sum_{|a-b|<\epsilon n}X_a[X_b(X_a+X'_a)X''_b+X_b X''_a(X_b+X'_b)+X'_b(X_a+X'_a)(X_b+X'_b)].
\end{align*}
Thus, we obtain a tractable expression of $Y$ in terms of i.i.d. Gaussian random vectors $(X,X',X'')$. We then upper bound the above expression using concentration inequalities for Gaussian vectors and matrices. The details are deferred to Appendix~\ref{app:ineqj}.
\section{Conclusion}
We studied convex relaxations of the quadratic assignment formulation for graph alignment under the correlated GOE model with correlation $1/\sqrt{1+\sigma^2}$. Focusing on relaxations over the Birkhoff polytope and a constrained hypercube, we established sharp recovery guarantees. In particular, we showed that when $\sigma = o(n^{-1/2}/\log^4 n)$, the solution of the relaxed problem concentrates around the ground-truth permutation matrix in Frobenius norm. Combined with existing lower bounds, our results precisely characterize the noise threshold at which $\|X^\star-\Pi^\star\|_F^2$ transitions from $o(n)$ to $\Omega(n)$. Our analysis further shows that the hypercube relaxation achieves the same guarantees as the Birkhoff relaxation, despite operating over a strictly larger feasible set. The resulting concentration bounds imply that simple post-processing methods recover all but a vanishing fraction of the true permutation. An important open direction is to develop a theoretical understanding of LAPs for $\sigma = \Omega(n^{-1/2})$ and to extend these guarantees to broader random graph models. \rev{A promising route is to establish diagonal dominance of the unregularized minimizer, as discussed in Section~\ref{sec:main_result}. Like Frobenius proximity to $\Pi^\star$, diagonal dominance guarantees that the post-processing of \eqref{eq:post} recovers all but a vanishing fraction of $\pi^\star$; however, whereas \cite{sushil25} showed that Frobenius proximity provably fails above $\sigma = n^{-1/2}$, diagonal dominance is expected to persist for substantially larger $\sigma$, in line with the empirical performance of the Birkhoff and simplex relaxations.}
\newpage
\bibliography{references}
\appendix
\section{Proof Details for Lemma~\ref{lem:relax}} \label{app:lemrelax}
\subsection{A concentration lemma}
We first need the following technical lemma concerning concentration of sums of i.i.d. random variables with zero mean and moment bounds.
\begin{lemma} \label{lemma:moment}
    Let $U \in \bbR^n$ be an i.i.d. random vector with $\dE(U)=0$ and $\dE(U_i^{2k}) \leq r_{k}$ for some constants $k, r_k>0$. Then there exists a constant $c_k$ depending only on $r_k$ and $k$ such that, with probability at least $1-c_k n^{-k/3}$,
    \begin{align*}
        \sum_{a=1}^n U_a \leq n^{2/3}.
    \end{align*}
\end{lemma}
\begin{proof}
We have
\begin{align*}
    \dE\left(\sum_{a=1}^n U_a\right)^{2k} &= \dE \left(\sum_{i_1, \hdots, i_{2k} =1}^n \prod_{j=1}^{2k} U_{i_j} \right).
    \numberthis \label{eq:moment_sum}
\end{align*}
Observe that, for any $(i_1,\dots,i_{2k})\in\{1,\dots,n\}^{2k}$, if $\{i_1,\dots,i_{2k}\}$ has cardinality strictly greater than $k$, then it must be that an index appears exactly once in the $2k$-uplet, hence, since $\dE(U)=0$,
\begin{equation*}
\dE \left(\prod_{j=1}^{2k} U_{i_j} \right) = 0.
\end{equation*}
Therefore, if we denote
\begin{equation*}
\mathcal{E}_{n,k}
=\left\{ E\subset \{1,\dots,n\}, \operatorname{Card}(E) \leq k\right\},
\end{equation*}
we have
\begin{align*}
    \dE\left(\sum_{a=1}^n U_a\right)^{2k}
    &\leq \sum_{E\in\mathcal{E}_{n,k}}
    \sum_{i_1, \hdots, i_{2k} \in E}
    \dE \left( \prod_{j=1}^{2k} U_{i_j} \right)
    &\overset{(*)}{\leq} \sum_{E\in\mathcal{E}_{n,k}}
    \operatorname{Card}(E)^{2k}
    \dE (U^{2k})
    &\overset{(**)}{\leq} n^k k^{2k} r_k,
\end{align*}
where $(*)$ follows as, by Hölder's inequality, for any set of iid random variables $Y_1,\dots,Y_s$ and any set of nonnegative integers $n_1,\dots,n_s$, it holds that
\begin{equation*}
\dE\left|\prod_{j=1}^s Y_j^{n_j}\right|
\leq \dE|Y_1^{n_1+\dots+n_s}|,
\end{equation*}
and $(**)$ is true because $\operatorname{Card}(E)\leq k$ for any $E\in\mathcal{E}_{n,k}$ and $\operatorname{Card}(\mathcal{E}_{n,k}) \leq n^k$.
Now, the Markov's inequality gives us
\begin{align*}
    \dP\left(\bigg|\sum_{a=1}^n U_a\bigg| > n^{2/3}\right) \leq n^{-4k/3} n^k k^{2k} r_k = r_k k^{2k} n^{-k/3}.
\end{align*}
This completes the proof.
\end{proof}
\subsection{The case of \texorpdfstring{$Y_{i, j}$}{} for \texorpdfstring{$i \neq j$}{}} \label{app:ineqj}
Each of the sums in \eqref{eq:y_as_function_of_x} can be split according to whether $a=b$ or $a\ne b$. The resulting sums when $a=b$ give rise to contributions of the following types:
\begin{align*}
S_0:&=\sum_{a=1}^n X_a X'_a (X''_a)^2,\; S_1:=\theta\sum_{a=1}^n X_a^2X'_a X''_a,\;S'_1:=\theta\sum_{a=1}^n X_a^2 (X''_a)^2,\\
S_2&:=\theta^2\sum_{a=1}^n X_a^3 X''_a,\; S_3:=\theta^3 \sum_{a=1}^n X_a^4,
\end{align*}
as well as sums of lesser order such as $\theta^3\sum_{a=1}^n X_a^3 X'_a$ and $\theta^3\sum_{a=1}^n X^2_a X'^2_a$. We start by upper bounding $S_0$ below:
\begin{align*}
    S_0 = \sum_{a=1}^n X_a X'_a (X''_a)^2 \leq O\left(\sqrt{\log n \sum_{a=1}^n (X'_a)^2 (X''_a)^4}\right) \leq O\left(\sqrt{n} (\log n)^{2}\right),
\end{align*}
where the first inequality holds w.h.p. $(1-n^{-c})$ by \cite[Lemma 13]{fan2023spectral} (conditioning on $X'$ and $X''$) and the second one by bounding each component of $X'$ and $X''$ by $O(\sqrt{\log n})$ w.h.p. $(1-n^{-c})$. The same argument gives w.h.p. $(1-n^{-c})$
\begin{align*}
    S_1 = O((\log n)^{2.5}), \; S_2 = O((\log n)^{3}/\sqrt{n}).
\end{align*}
Next, with probability at least $1-n^{-c}$, we have
\begin{align*}
    \sum_{a=1}^n X_a^2 (X''_a)^2 &= n + \sum_{a=1}^n (X_a^2-1) + \sum_{a=1}^n ((X''_a)^2-1) + \sum_{a=1}^n (X_a^2-1)((X''_a)^2-1) \\
    &\leq n + O(n^{2/3}) \quad \text{by Lemma~\ref{lemma:moment}},
\end{align*}
which ensures $S_1' = O(\sqrt{n \log n})$. Similar argument gives $S_3 = O((\log n)^{3/2}/\sqrt{n})$.
So, letting $Y=Y_{=}+Y_{\ne}$ where the subscript corresponds to summations with $a=b$ or $a\ne b$, recalling that $\alpha \beta =\Omega(n)$, we find
$$
Y_==O\left((\log n)^{2}/\sqrt{n}\right).
$$
It remains to control $Y_{\ne}$. Taking into account symmetries between $X$, $X'$, $X''$ and between $a$ and $b$, denoting by $S$ the summation over $a$, $b$ in $[n]$ such that $a\ne b$ and $|a-b|<\epsilon n$, we find (omitting the sums of lesser order as we did for $Y_=$) that
\begin{align*}
Y_{\ne}=O\bigg(\frac{1}{n}\max\big[S(X_a&X_bX'_aX''_b), \theta S(X_a^2X'_b X''_b), \theta S(X_aX_b X'_a X'_b), \\
&\theta^2S(X_a^2X_bX'_b),\theta^2S(X_a^2(X'_b)^2),\theta^3 S(X_a^2 X_b^2)\big]\bigg).
\end{align*}
\paragraph{Control of $S(X^2_aX^2_b)$:}
Rewrite it as
$$
S(1)+S(X^2_a-1)+S(X^2_b-1)+S((X^2_a-1)(X^2_b-1)).
$$
The first term $S(1)$ is $(1+o(1))n^2\epsilon$. The second and third term are weighted sums of i.i.d. centered random variables with weights $(1+o(1))(2n\epsilon)$, and Chernoff bounds entail that these are with probability $1-O(n^{-c})$ equal to $O(n\epsilon\sqrt{n\ln n})$. The third term is a quadratic form of a centered i.i.d. random vector with sub-exponential tails. Thus by Proposition 1.1 in \cite{10.1214/21-EJP606}
 we have
$$
\dP(|S(X^2_a-1)S(X^2_b-1))|\ge t)\le 2 \exp\left(-\frac{1}{C}\min\left(\frac{t^2}{\|M\|^2_F},\sqrt{\frac{t}{\|M\|_{op}}}\right)   \right),
$$
where $C$ is a universal constant and  $M_{ab}=\mathbf{1}_{a\ne b, |b-a|<\epsilon n}$. Here we have $\|M\|_{op}\le 2n \epsilon$, $\|M\|^2_{F}\le 2 n^2 \epsilon$, and we find that by taking $t=\Omega(n\epsilon \ln^2 n)$, we obtain a right-hand side in $O(n^{-\Omega(1)})$.
Altogether, with probability $1-O(n^{-c})$ one has $S(X^2_aX^2_b)=O(n^2 \epsilon)$, and thus its contribution to the upper-bound on $Y_{\ne}$ is $O(n^{-1}\theta^3 n^2\epsilon)=O(\ln^{3/2}n/n^{1/2})$.
\paragraph{Control of $S(X^2_aX'^2_b)$:}
Rewrite it as
$$
S(1)+S(X^2_a-1)+S(X'^2_b-1)+S((X^2_a-1)(X'^2_b-1)).
$$
We can apply the same bounds as in the previous paragraph, and we obtain that the contribution to the upper-bound on $Y_{\ne}$ is with probability $1-O(n^{-c})$ in $O(n^{-1}\theta^2 n^2 \epsilon)=O(\epsilon \ln n)$.
\paragraph{Control of $S(X_aX'_aX_bX'_b)$:} Write $\tilde{X}_a=X_aX'_a$. Thus we are again evaluating $S(\tilde{X}_a \tilde{X}_b)$, where here $\tilde{X}_a$ are i.i.d., zero-mean, and sub-exponential. Relying again on Proposition 1.1 in \cite{10.1214/21-EJP606} we have
$$
\dP(|S(\tilde{X}_a\tilde{X}_b)|\ge t)\le 2\exp\left(-\frac{1}{C}\min\left(\frac{t^2}{\|M\|^2_F},\sqrt{\frac{t}{\|M\|_{op}}}\right)   \right),
$$
and for $t=\Omega(n\epsilon \ln^2 n)$ we find that with probability $1-O(n^{-c})$, $S(Y_aY_b)=O(n\epsilon\ln^2n)$. Thus the contribution to the upper bound on $Y_{\ne}$ is $O(n^{-1}\theta n\epsilon \ln^2n)=O(\ln^{3/2}n/n^{1/2})$.
\paragraph{Control of $S(X^2_aX_bX'_b)$:}
This admits, conditionally on $X$, a Gaussian distribution with variance $\sum_b \left( X_b \sum_{a\ne b,|b-a|<n\epsilon}X^2_a\right)^2$. Thus this is, with probability $1-O(n^{-c})$, no larger than
$$
\sqrt{\ln (n)} \sqrt{\sum_b \left( X_b \sum_{a\ne b,|b-a|<n\epsilon}X^2_a\right)^2}.
$$
We can bound with high ($1-O(n^{-c})$) probability each $X_a$ by $O(\sqrt{\ln n})$, giving an upper bound in
$$
O(\sqrt{\ln (n)} \sqrt{n \epsilon^2\ln^3n n^2})=O(\epsilon n^{3/2}\ln^2 n).
$$
Thus we get an overall contribution to the upper-bound on $Y_{\ne}$ bounded by $O(n^{-1}\theta^2 n^{3/2}\ln^2 n)=O(\ln^3n/n^{1/2})$.
\paragraph{Control of $S(X^2_aX'_bX''_b)$:} Conditionally on $X,X'$ this is a Gaussian random variable with variance $\sum_b \left(X'_b\sum_{a\ne b, |b-a|<n\epsilon}X_a^2\right)^2$. Again with high probability ($1-o(n^{-c})$) we obtain the upper bound
$$
O\left(\sqrt{\ln n}\sqrt{\sum_b \left(X'_b\sum_{a\ne b, |b-a|<n\epsilon}X_a^2\right)^2}\right)
$$
Upper-bounding the entries of $X$ and $X'$ by $O(\sqrt{\ln n})$ this is also
$$
O\left((\ln^2 n)\sqrt{n (n\epsilon)^2}\right)=O((\ln^2 n)\epsilon n^{3/2}).
$$
The contribution to the upper-bound of $Y_{\ne}$ is then $O(n^{-1}\theta (\ln^2 n) \epsilon n^{3/2})=O(\epsilon \ln^{5/2}n)$.
\paragraph{Control of $S(X_aX_bX'_aX''_b)$:} Conditionally on $X,X'$ this is Gaussian with variance
$$\sum_b \left(X_b\sum_{a\ne b:|b-a|<n\epsilon}X_a X'_a\right)^2.
$$
The term $X_b\sum_{a\ne b:|b-a|<n\epsilon}X_a X'_a$ is Gaussian with variance $V_b:=X_b^2\sum_{a\ne b:|b-a|<n\epsilon}X_a^2$ conditionally on $X$. This variance term $V_b$ is, with high probability ($1-o(n^{-c})$), $O(\epsilon n \ln^{2}n)$, so that $\left(X_b\sum_{a\ne b:|b-a|<n\epsilon}X_a X'_a\right)^2$ is, with high probability ($1-o(n^{-c})$), $O(\epsilon n \ln^{3}n)$. Thus with high probability ($1-o(n^{-c})$), $S(X_aX_bX'_aX''_b)$ is
$$
O\left(\sqrt{\ln n} \sqrt{\epsilon n^2 \ln^3n}\right)=O(\sqrt{\epsilon} n \ln^2n).
$$
Thus the contribution to the upper-bound on $Y_{\ne}$ is $O(n^{-1}\sqrt{\epsilon}n\ln^2n)=O(\sqrt{\epsilon} \ln^2n)$.
Of all these 6 terms, the last one is the largest. By a union bound, assuming we took $c>2$ in our control in $(1-o(n^{-c}))$ of the probability of desired events, we thus have $|Y_{ij}|\le O(\sqrt{\epsilon}\ln^2n)$ for all $i,j\in [n]$, $i\ne j$.
Thus it follows that with high probability, for all $i\ne j\in [n]$, $|Y_{ij}|\le 1/2$ provided  $\epsilon\le \kappa /\ln^4n$ for sufficiently small (but still $\Omega(1)$) constant $\kappa>0$.
\subsection{The case of \texorpdfstring{$Y_{i, j}$}{} for \texorpdfstring{$i = j$}{}} \label{sec:iequalj}
\paragraph{Bounding $Y_{kk}$:} This will follow the same path as for bounding $Y_{ij}$, while being slightly simpler. Write
$$
Y_{kk}=\sum_{a,b:|b-a|<n\epsilon}(e_a^\top Ve)(e_b^\top Ve)(e_a^\top V e_k)(e_b^\top V e_k).
$$
We use two Gaussian vectors $X,X'$ as before to form $V e_k=\alpha^{-1} X$, and
$$
V e=\alpha^{-1} X +\beta^{-1}\sqrt{n-1}(X'-r\alpha^{-1} X),
$$
so that
\begin{align*}
\alpha^2 Y_{kk}={}& \sum_{a,b:|b-a|<n\epsilon}X_a X_b \times \\
&\left(\left[\frac{1}{\alpha}-\frac{\sqrt{n-1} r}{\alpha \beta}\right]X_a+\frac{\sqrt{n-1}}{\beta}X'_a\right)\left(\left[\frac{1}{\alpha}-\frac{\sqrt{n-1} r}{\alpha\beta}\right]X_b+\frac{\sqrt{n-1}}{\beta}X'_b\right)
\\
={}& \sum_{a,b:|b-a|<n\epsilon}X_a X_b (\theta X_a +(1+o(1))X'_a)(\theta X_b + (1+o(1))X'_b)\\
={}& (1+o(1))\sum_{a,b:|b-a|<n\epsilon}X_a X_b\left[X'_a X'_b+\theta (X_a X'_b+X'_a X_b)+\theta^2 X_a X_b\right],
\end{align*}
where as before $\theta=O(\sqrt{\ln n/n})$.
We first evaluate the sums over $a=b$, which read
$$
S_0=\sum_a X_a^2(X'_a)^2, \; S_1=\theta \sum_a X_a^3 X'_a,\; S_2=\theta^2\sum_a  X_a^4.
$$
By Lemma~\ref{lemma:moment}, the first sum can be written
$$
S_0=n+\sum_a X^2_a(X'^2_a-1) +\sum_a (X^2_a-1) = n + O(n^{2/3}).
$$
Thus, with high probability $1-o(n^{-c})$, $S_0=n + O(n^{2/3})$. The same argument entails that with high probability ($1-o(n^{-c})$), $S_1=O(\theta n^{2/3})=O(\sqrt{\ln n} n^{1/6})$, $S_2=O(\theta^2 n)=O(\ln n)$.
We decompose again $Y_{kk}$ as $Y_=+Y_{\ne}$. The previous evaluations together with $\alpha^2=n(1+o(1))$ give us
$$
Y_==1+o(1).
$$
Next we evaluate $Y_{\ne}$. We readily have:
$$
Y_{\ne}=O\left(\frac{1}{n}\max\left[S(X_a X_b X'_a X'_b),\theta S(X^2_a X_b X'_b),\theta^2 S(X^2_aX^2_b)\right]\right).
$$
As previously established, $S(X^2_aX^2_b)=O(n^2\epsilon)$, and the corresponding contribution to the upper-bound on $Y_{\ne}$ is $O(n^{-1}\theta^2 n^2\epsilon)=O(\epsilon \ln n)$. We also established $S(X^2_a X_b X'_b)$ is $O(\epsilon n^{3/2}\ln^2n)$. Thus its contribution to the upper-bound on $Y_{\ne}$ is $O(n^{-1}\theta \epsilon n^{3/2}\ln^2 n)=O(\epsilon \ln^{5/2}n)$. Finally, we have established that $S(X_a X_b X'_a X'_b)=O(n\epsilon \ln^2 n)$. This term thus gives a contribution in $O(\epsilon \ln^2n)$ to the upper bound. Summarizing, we find that $Y_{\ne}=O(\epsilon \ln^{5/2}n)$. Under the previously identified condition that $\epsilon\le \kappa/\ln^4 n$, this is $o(1)$. For this regime, we thus have that
$$
Y_{kk}=1+o(1),
$$
which concludes the proof of Lemma \ref{lem:relax}.
\section{Proof of Lemma \ref{lemma:2}.} \label{app:lemma2}
Introduce the semi-circle density
$$
\forall x\in \dR,\quad \rho_{sc}(x):=\frac{1}{2\pi}\sqrt{(4-x^2)_+},
$$
and the so-called {\em classical location} $\nu_j$ of eigenvalue $\lambda_j$ of the Gaussian Wigner matrix $A_2$, defined through
$$
\int_{-\infty}^{\nu_j}\rho_{sc}(x)dx=\frac{j}{n}\cdot
$$
The main result in Erd\H{o}s, Yau and Yin \cite{ERDOS20121435} implies the existence of some constant $C=\Theta(1)$ such that with high probability,
\begin{equation}\label{eq:rigidity}
\forall j\in [n],\quad |\lambda_j -\nu_j|\le \left( \ln n\right)^{C \ln(\ln n)} n^{-2/3}.
\end{equation}
In fact the results of \cite{ERDOS20121435} give a stronger guarantee, a precision on the probability of the ``rigidity property'' to fail, and apply beyond the case of Gaussian Wigner matrices, but their implication \eqref{eq:rigidity} is sufficient for our purpose.
Note that $\rho_{sc}(x)$ is upper-bounded by $1/\pi$ on $\dR$, so that
$$
\forall i,j\in[n],\; i<j,\; \int_{\nu_i}^{\nu_j}\rho_{sc}(x)dx\le \frac{1}{\pi}(\nu_j-\nu_i).
$$
By definition of the classical locations $\nu_j$, this entails
$$
\forall i,j\in[n],\; i<j,\; \nu_j-\nu_i\ge \pi \frac{j-i}{n}\cdot
$$
Recall Lemma \ref{lemma:2}'s assumption that $\epsilon\ge n^{\delta-2/3}$ for some $\delta=\Omega(1)$. The term $\left( \ln n\right)^{C \ln(\ln n)}$ in \eqref{eq:rigidity} is less than $n^{\delta/2}$ for large enough $n$. It then follows from \eqref{eq:rigidity} and the previous display that
$$
\begin{array}{ll}
\forall i,j\in[n],\; i<j,\; \lambda_j-\lambda_i&\ge \nu_j-\nu_i -2 n^{\delta/2-2/3}\\
&\ge \pi\frac{j-i}{n}-2n^{\delta/2-2/3}.
\end{array}
$$
This yields
$$
\begin{array}{ll}
C_\epsilon&=\min_{j-i\ge \epsilon n}(\lambda_j -\lambda_i)^2\\
&\ge \min_{i}(\lambda_{i+\epsilon n}-\lambda_i)^2\\
&\ge \left(\pi\epsilon -2 n^{\delta/2-2/3}\right)^2\\
&\ge \left((\pi-1)\epsilon + \Omega(n^{\delta-2/3})-2n^{\delta/2-2/3}\right)^2\\
&\ge (\pi-1)^2\epsilon^2
\end{array}
$$
for sufficiently large $n$, which concludes the proof.
\end{document}